%% file: latex/main.tex
\documentclass[11pt]{article}

\usepackage[final]{acl}

\usepackage{tcolorbox}
\usepackage{times}

\usepackage[dvipsnames]{xcolor}
\usepackage{latexsym}
\usepackage{enumitem}
\usepackage{amsthm}
\usepackage{amsfonts}
\usepackage{setspace}
\usepackage[T1]{fontenc}

\usepackage[utf8]{inputenc}

\usepackage{microtype}
\usepackage{algorithm}
\usepackage{algorithmic}
\usepackage{amssymb}
\usepackage{pifont}

\usepackage{comment}

\usepackage{inconsolata}

\usepackage{graphicx}
\usepackage{amssymb}
 \usepackage{amsmath} 
 \usepackage{comment}

 \usepackage{multirow}
\usepackage{booktabs}
\usepackage{hyperref}

\usepackage{colortbl}
\definecolor{promptgray}{RGB}{245,245,245}
\definecolor{toolblue}{RGB}{235,242,250}
\definecolor{faithgreen}{RGB}{236,246,240}
\newtheorem{corollary}{Corollary}
\newtheorem{lemma}{Lemma}

\title{Faithful-First Reasoning, Planning, and Acting for Multimodal LLMs}

\author{
  Junxian Li\textsuperscript{1}\thanks{Equal contribution.} \quad
  Xinyue Xu\textsuperscript{2}\footnotemark[1] \quad
  Sai Ma\textsuperscript{3}\footnotemark[1] \quad
  Di Zhang\textsuperscript{4} \quad
  Sichao Li\textsuperscript{5}\thanks{Correspondence: \href{mailto:sichao.li@sydney.edu.au}{sichao.li@sydney.edu.au}}
  \\
  \textsuperscript{1}Shanghai Jiao Tong University \\
  \textsuperscript{2}The Hong Kong University of Science and Technology \\
  \textsuperscript{3}The Australian National University 
    \textsuperscript{4} Fudan University 
  \textsuperscript{5}University of Sydney \\
  \texttt{lijunxian0531@sjtu.edu.cn, xinyue.xu@connect.ust.hk, sai.ma@anu.edu.au} \\ \texttt{di.zhang@ustc.edu, sichao.li@sydney.edu.au}
}

\usepackage{xspace}

\newcommand{\dataset}[1]{\textsc{#1}\xspace}

\newcommand{\model}[1]{\textsc{#1}\xspace}

\newcommand{\code}[1]{\texttt{#1}\xspace}

\begin{document}
\maketitle


\begin{abstract}
Multimodal Large Language Models (MLLMs) frequently suffer from unfaithfulness, generating reasoning chains that drift from visual evidence or contradict final predictions. We propose \emph{Faithful-First Reasoning, Planning, and Acting (RPA) framework} in which \textsc{FaithEvi} provides step-wise and chain-level supervision by evaluating the faithfulness of intermediate \emph{reasoning}, and \textsc{FaithAct} uses these signals to \texttt{plan and execute} faithfulness-aware actions during inference.
Experiments across multiple multimodal reasoning benchmarks show that faithful-first RPA improves perceptual faithfulness by up to \textbf{24\%} over \emph{prompt-based} and \emph{tool-augmented} reasoning frameworks, without degrading task accuracy. 
Our analysis shows that treating \emph{faithfulness as a guiding principle} perceptually faithful reasoning trajectories and mitigates hallucination behavior.
This work thereby establishes a unified framework for both evaluating and enforcing faithfulness in multimodal reasoning. Code is at \url{https://github.com/lijunxian111/Faithful-First-RPA}.
\end{abstract}

\section{Introduction}
Despite rapid progress in multimodal large language models (MLLMs)~\cite{2023llavarlhf, chen2024expanding, an2025llava, wang2024qwen2}, their reasoning traces \emph{remain unfaithful}: models frequently produce persuasive explanations that conflict with perceptual evidence, or utilize post-hoc rationalizations fabricate their reasoning progress \cite{arcuschin2025chain, barez2025chain}. This gap poses a central challenge for trustworthy reasoning. Existing efforts typically focus on improving task accuracy or enriching \textsc{Chain-of-Thought} (CoT) generation~\cite{zhang2023multimodal}, yet the unfaithful reasoning remains unaddressed.

\paragraph{Motivation.}
We are motivated by the following principle and observations \cite{goyal2017making}.
\begin{quote}
    \textit{A perceptually faithful model reasons only over what is visually observable; it does not ``see'' what the image does not reveal.}
\end{quote}
\noindent This principle echoes long-standing findings in Visual Question Answering (VQA): systems should avoid answering beyond available evidence and resist over-reliance on language priors \cite{antol2015vqa, agrawal2018don, bender2021stochastic}.
\begin{figure*}[t]
    \centering
    \includegraphics[width=\linewidth]{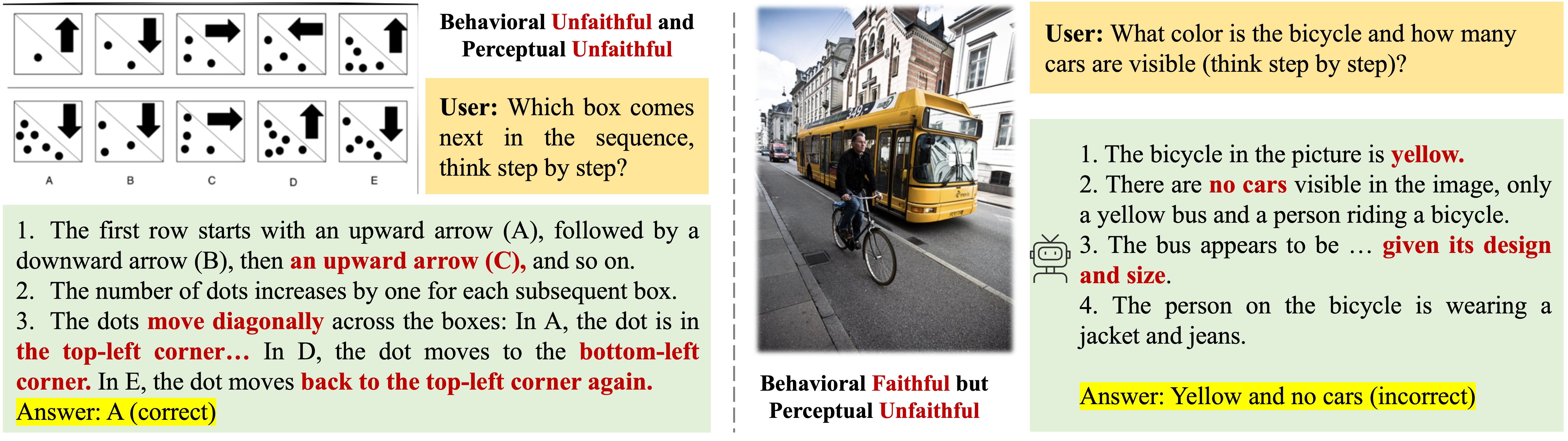}
    \caption{Perceptually and behaviorally unfaithful examples. \textit{Left}: \textbf{Perceptually unfaithful and behaviorally unfaithful.} The model presents a  step-by-step reasoning trace describing alternating arrow directions and increasing dot counts, yet such reasoning does not reflect its actual decision process. The final choice (A) is likely made through pattern association, with the explanation generated post hoc to rationalize it. \textit{Right}: \textbf{Perceptually unfaithful but behaviorally faithful}. The reasoning aligns with its final answer but incorrectly describes the bicycle as yellow, influenced by a nearby yellow bus. This illustrates a visually plausible but perception-unfaithful reasoning step, where linguistic association overrides perceptual grounding.
}
    \label{fig:motivation}
\end{figure*}

As illustrated by examples in Fig.~\ref{fig:motivation}, we observe that MLLMs can generate plausible explanations that are \textbf{perceptually inconsistent} with the underlying visual input, \emph{regardless of whether the final prediction is correct}. To formalize this observation, we 
distinguish perceptual faithfulness (reasoning steps align with the model's input) from the predominant focus in prior work on behavioral faithfulness (reasoning steps align with the model's output)  
\cite{arcuschin2025chain, matton2025walk, ming2024faitheval,li2023evaluating}. 



\noindent {\textbf{Our answer.}} We argue that \textsc{faithfulness should be a design principle, not merely a post-hoc evaluation objective}. Reasoning frameworks should \emph{explicitly verify} the evidential grounding of each step before it is admitted into the chain, ensuring that reasoning remains both perceptually grounded and behaviorally aligned.

\paragraph{Contribution.}
This motivates the \emph{Faithful-First Reasoning, Planning, and Acting (RPA)} framework, which operationalizes the principle as a design constraint, enforcing faithfulness throughout multimodal reasoning.
Within this framework, we make the following contributions:

 \noindent $\bullet$ \noindent We introduce \textsc{FaithEvi}, an evidence-based perceptual faithfulness evaluation pipeline via evidence that extracts claimed objects from each reasoning step and verifies their existence at both local and global levels through preference polling and visual grounding. It assigns a step-level faithfulness score ($F_{\text{step},t}$) to each reasoning step and aggregates them into a chain-level score ($F_{\text{chain}}$), providing a general quantitative measure of perceptual faithfulness with respect to the input evidence.


\noindent $\bullet$ We propose \textsc{FaithAct}, a faithful-first planning and acting mechanism that enforces evidential grounding during inference. It operates through two sequential phases: (1) \emph{Planning}, which leverages \textsc{FaithEvi} signals to assess the current chain's faithfulness ($F_{\emph{chain}}$) and determine admissible next steps; and (2) \emph{Acting}, which executes faithfulness-aware actions subject to dynamic thresholds. The mechanism is realized via a lightweight, extensible interface of callable functions.
    


\noindent $\bullet$ Extensive experiments across multiple benchmarks reveal that current MLLMs consistently underestimate evidential support in reasoning process. We demonstrate that enforcing step-wise verification significantly enhances perceptual faithfulness without compromising task accuracy. Compared to standard prompt-based and tool-augmented frameworks, \emph{Faithful-First RPA} achieves the highest overall perceptual faithfulness and effectively mitigates content hallucination.
    

\section{Perceptual and Behavioral Faithfulness}

Building on the motivating observations (Fig.~\ref{fig:motivation}) and prior analyses of explanation faithfulness (Appx.~\ref{sec:related_frameworks}), we distinguish two ``unfaithful'' notions and clarify our focus of this work.

Previous studies mainly concern whether a reasoning trace accurately reflects the decision process that led to the model's final output \cite{lyu2023faithful, matton2025walk, barez2025chain, arcuschin2025chain,li2026textbf,ma2026thinking}, which we refer to as \textbf{Behavioral faithfulness (BF)}. 
While important, it underestimates the importance of input evidence in the reasoning process. We therefore propose \textbf{Perceptual faithfulness (PF)}, which concerns whether individual reasoning steps are grounded in verifiable input evidence, such as entities and attributes that are perceptually present in the input. A perceptually unfaithful explanation may invoke objects and properties that are unsupported or contradicted by the visual input.
Importantly, the model is not required to output the correct answer to be faithful \citep{jacovi2020towards, dasgupta2022framework}. 

For instance, the left example in Fig.~\ref{fig:motivation} shows a correct prediction accompanied by reasoning that is neither perceptually grounded nor behaviorally aligned. 
The right example illustrates reasoning that aligns with the model's prediction but remains perceptually unfaithful, as it relies on visually unsupported attributes (e.g., bicycle color). While ideal reasoning would be both perceptually and behaviorally faithful and yield a correct prediction, optimizing all three remains challenging. 
In this work, we focus on \emph{perceptual faithfulness}. 

This choice is motivated by the hypothesis that when reasoning remains perceptually grounded, behavioral consistency tends to follow, without introducing additional optimization objectives such as output correctness. Practically, perceptual faithfulness offers a more actionable and operationally measurable objective, each reasoning step can be directly validated against the available multimodal evidence by inference.
In contrast, behavioral faithfulness depends on inaccessible internal dynamics of MLLMs and is  therefore treated as a consequence rather than a controllable design target.


\section{\textsc{FaithEvi}: Perceptual Faithfulness Evaluation Pipeline}
\begin{figure*}[t]
    \centering
    \includegraphics[width=\linewidth]{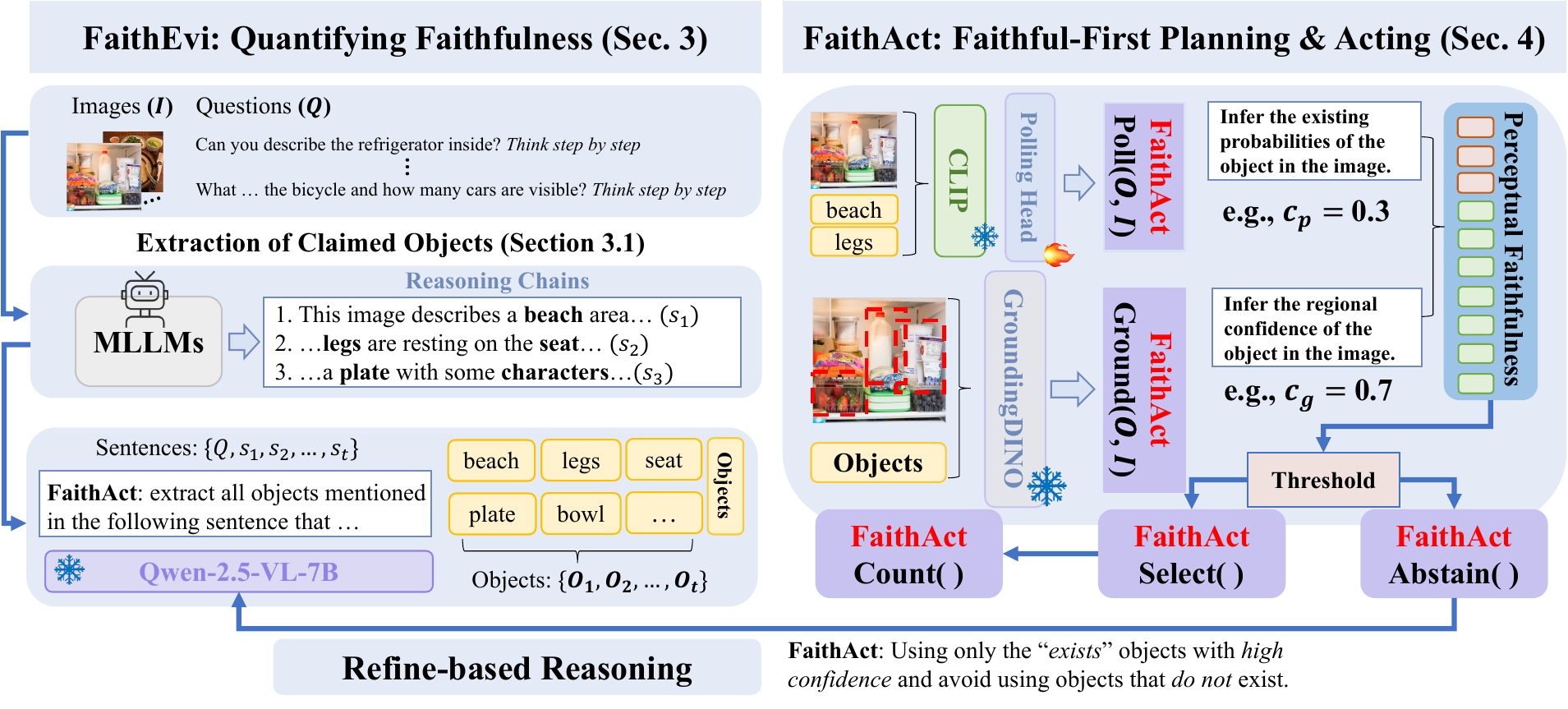}
    \caption{\textbf{Faithful-first reasoning, planning, and acting framework.} Given an image-question pair, \textsc{FaithEvi} evaluates the perceptual faithfulness of intermediate \textbf{reasoning}, producing step- and chain-level faithfulness scores. Guided by these signals, \textsc{FaithAct} \emph{plans} and \emph{acts} faithfulness-aware actions during inference.}
    \label{fig:overall_framework}
\end{figure*}

\noindent We operationalize perceptual faithfulness by first evaluating how well each reasoning step aligns with the input verifiable evidence. 
To this end, we design a general-purpose pipeline that systematically evaluates the degree of perceptual grounding across reasoning process at step-wise and chain-level.
The pipeline consists of three key stages: (i) \emph{Extraction of Claimed Objects} (ii) \emph{Preference Polling and Grounding} (iii) \emph{Faithfulness Scoring}. 

\subsection{Extraction of Claimed Objects}
\label{subsec:objectext}
We begin by defining the setting for perceptual faithfulness verification.
Given a multimodal input consisting of an image $I$ and and a corresponding textual query $Q$, an MLLM produces a reasoning chain by prompting ``think step by step''~\cite{wei2022chain}: $R_{raw} = \{ s_1, s_2, \dots, s_T \},$
where each \( s_t \) represents an intermediate reasoning step generated by the model.  
To quantify perceptual faithfulness, we must identify all \emph{claimed objects} across both the question and the reasoning steps.

We therefore employ a structured extraction process to isolate visually grounded claims from textual noise.  
Specifically, we process the concatenated text of \( Q \) and each reasoning step \( s_t \) using a helper LLM to extract meaningful object mentions.  
Each reasoning step is independently analyzed, excluding introductory or concluding phrases (e.g., ``Let's begin reasoning'' or ``Therefore, the answer is...'') to focus solely on evidence-bearing content.

Concretely, we query \code{Qwen2.5-7B-Instruct} (Qwen) ~\cite{yang2407qwen2} with a structured prompt (see Appx.~\ref{appendix:object_ext_prompt}) to extract, for each step \( s_t \), a set $O_t$ of $m_t$ claimed objects, where $O_t = \{ O_t^1, O_t^2, \dots, O_t^{m_t} \}.$
The union of all extracted objects across the reasoning chain is denoted as
\begin{equation}
    O = \bigcup_{t=1}^{T} O_t.
\end{equation}

Each \( O_t^i \) represents a semantically meaningful object or concept that may correspond to perceivable entities in the input image.  
This step ensures that the subsequent grounding and polling stages operate only on visually relevant content, filtering out abstract or non-visual reasoning tokens.  
The extracted object sets \( \{O_t\}_{t=1}^T \) thus serve as the foundation for our evidence verification pipeline, linking textual reasoning to visual perception.

\subsection{Preference Polling and Grounding}
Once the claimed objects are extracted, the next step is to verify their perceptual validity in the input image.  
This verification involves evaluating both the \emph{existence} and \emph{localization} of each claimed object through a two-stage evidence assessment pipeline: 
(i) \textbf{Preference Polling}, which estimates the likelihood that a claimed object is perceptually verifiable in the image, and  
(ii) \textbf{Grounding}, which localizes verified objects to specific visual regions for spatial confirmation.  
 
\subsubsection{Preference Polling}
\label{subsec:poll}

While object detectors can localize objects spatially, their confidence scores are often unreliable when visual cues are weak, such as in cases of occlusion, clutter, or low salience.  
To address this limitation, we introduce a lightweight \emph{preference polling} module that performs an initial verification of object existence prior to grounding. 
This module serves as an evidence gate, polling the visual scene to estimate whether an object is perceptually verifiable.  

We employ a frozen \code{CLIP-ViT-Large} (CLIP) model~\cite{radford2021learning} as the multimodal encoder pair $(f_{\text{img}}, f_{\text{txt}})$ and attach a lightweight polling head, a two-layer MLP with GELU activation, to predict whether a given object exists in the image. The polling model is trained on the POPE dataset~\cite{li2023evaluating}, which provides binary image–object existence labels (see Appx.~\ref{appendix:clip_training}).

Given an image \( I \) and an object \( X \), CLIP encoders produce visual and textual embeddings:
\begin{equation}
    \mathbf{v} = f_{\text{img}}(I) \in \mathbb{R}^{d}, 
    \quad 
    \mathbf{t} = f_{\text{text}}(X) \in \mathbb{R}^{d}.
\end{equation}
The element-wise product \( \mathbf{v} \odot \mathbf{t} \) captures cross-modal interaction, which is passed through the polling head to yield a confidence score:
\begin{equation}
    c_{p} = \sigma \left( \mathbf{W}_2 \, \mathrm{GELU}\left(\mathbf{W}_1 (\mathbf{v} \odot \mathbf{t})\right) \right),
\end{equation}
where \( \mathbf{W}_1 \in \mathbb{R}^{512 \times d} \), \( \mathbf{W}_2 \in \mathbb{R}^{2 \times 512} \), and \( \sigma \) denotes the sigmoid function.  
The resulting confidence \( c_{p} \in [0,1] \) is interpreted as the probability that object \( X \) exists in the image.

For each reasoning step \( s_t \), polling produces a set of existence confidence scores:
\begin{equation}
    C_{p,t} = \{ c_{p,t}^1, c_{p,t}^2, \dots, c_{p,t}^{m_t} \},
\end{equation}
where \( c_{p,t}^i \) corresponds to the predicted existence probability for each claimed object \( O_t^i \).  
This early verification step thus acts as a perceptual verifier, aligning with our evidence-first philosophy.

\subsubsection{Grounding}
\label{subsec:grounding}

After preference polling confirms the likely existence of an object, we proceed to localize it within the image using a grounded detection model. This step provides region-level visual evidence and confidence scores for objects.

For each object \( O_t^i \) we employ a frozen \code{GroundingDINO-base} model~\cite{liu2023grounding} to detect and ground the object in the input image \( I \).  
\code{GroundingDINO} returns a set of bounding boxes and associated confidence scores:
\begin{align} 
    & B_t^i = \{ b_t^{i,1}, b_t^{i,2}, \dots, b_t^{i,k_i} \}, \nonumber \\ 
    \quad 
    & C_t^i = \{ c_t^{i,1}, c_t^{i,2}, \dots, c_t^{i,k_i} \}.
\end{align}
Each \( b_t^{i,j} \) represents a candidate region corresponding to object \( O_t^i \), and \( c_t^{i,j} \in [0,1] \) denotes the model’s confidence that the region indeed contains the object.  
We retain the most confident detection:
\begin{equation}
    c_{g,t}^i = \max_j c_t^{i,j}.
\end{equation}
The resulting set of grounded confidence scores for step \( s_t \) is:
\begin{equation}
    C_{g,t} = \{ c_{g,t}^1, c_{g,t}^2, \dots, c_{g,t}^{m_t} \}.
\end{equation}

\paragraph{Remark.}
While preference polling provides global existence verification, grounding supplies fine-grained spatial evidence. Combining these complementary signals yields a more reliable perceptual faithfulness measurement.

\subsection{Faithfulness Scoring}
\label{subsec:judge}
After obtaining the \emph{polling confidence} set \( C_{p,t} \) and \emph{grounding confidence} set \( C_{g,t} \) for each reasoning step \( s_t \), we combine them to compute the perceptual faithfulness of the reasoning process.

\paragraph{Object-Level Confidence.}
For each claimed object \( O_t^i \) within step \( s_t \), we compute an overall confidence score \( c_t^i \) by fusing the existence confidence from preference polling and the spatial confidence from grounding:
\begin{align}
    & c_t^i = \alpha \, c_{p,t}^i + (1-\alpha) \, c_{g,t}^i,
    \quad \nonumber \\
    & c_{p,t}^i \in C_{p,t}, \;
    c_{g,t}^i \in C_{g,t},
\end{align}
where \( \alpha \in [0,1] \) controls the relative importance of existence versus localization confidence.  
Since the polling confidence directly reflects perceptual existence, we empirically set \( \alpha = 0.7 \).

To interpret the confidence \( c_t^i \) as a discrete faithfulness score, we define a three-level mapping function \( f_t^i \) as:
\[ 
f_t^i =
\begin{cases}
    0, & c_t^i < 0.4, \quad \text{(confidently absent)} \\
    c_t^i, & 0.4 \leq c_t^i \leq 0.6, \text{(uncertain existence)} \\
    1, & c_t^i > 0.6, \quad \text{(confidently present)}.
\end{cases}
\]

\paragraph{Step-Level and Chain-Level Faithfulness.}
We then aggregate the object-level confidence scores to obtain a \emph{step-level} faithfulness score for each reasoning step \( s_t \):
\begin{equation}
    F_{\text{step},t} = \frac{1}{m_t} \sum_{i=1}^{m_t} f_t^i,
\end{equation}
which quantifies how faithfully the specific step's visual claims align with the image.  
Finally, we compute the \emph{chain-level} perceptual faithfulness of the entire reasoning process as the mean of the verified step-level scores:
\begin{equation}
    F_{\text{chain}} = \frac{1}{n} \sum_{t=1}^{n} F_{\text{step},t},
\end{equation}
where \( n \) denotes the total number of reasoning steps (excluding introductory or concluding phrases).  

A higher \( F_{\text{chain}} \) indicates that the reasoning chain is more consistently grounded in visual evidence, while lower scores suggest the presence of unsupported object references.  
This quantitative formulation allows us to evaluate perceptual faithfulness both locally (per step) and globally (across the reasoning trace), serving as a \emph{general} perceptual faithfulness evaluation pipeline.

\section{\textsc{FaithAct}: Faithful-First Planning and Acting}
Having established how to quantify perceptual faithfulness, we now turn to the problem of 
\emph{how to integrate faithfulness into the reasoning process itself}. We propose \textsc{FaithAct}, a \emph{\textbf{Faith}fulness-First Planning and \textbf{Act}ing} framework that enforces evidential verification during reasoning generation. Unlike conventional \emph{generate-then-verify} paradigm, \textsc{FaithAct} follows a \emph{verify-as-you-generate} principle, where each step is explicitly checked for perceptual support before being admitted into the reasoning chain. 


\subsection{Planning Objective}
We explicitly formulate reasoning process as a faithfulness-constrained planning problem: Given a query and image, the planner seeks a sequence of faithful reasoning steps:
\begin{equation} 
    S^{*} = \arg \text{max} F_{step}(s_t) \quad \text{s.t.} \forall t \; F_{step}(s_t) \geq c,
\end{equation}
where $s_{t}$ is the $t$-th reasoning step, $F_{step}(s_t)$ is its faithfulness score from Sec~\ref{subsec:judge}, and $c$ is a minimum evidential confidence threshold.
This converts reasoning into a faithfulness-regularized plan process: only steps sufficiently supported by evidence are eligible to advance the reasoning chain. When a proposed step \( s_t \) fails to meet the threshold (\( F_{\text{step}}(s_t) < c \)), the planner either refines or regenerates the step before proceeding, thereby enforcing the \emph{faithfulness desideratum}.

\subsection{Planning-and-Acting Loop}
\label{subsec:calling}
\textsc{FaithAct} executes a \emph{faithfulness-first planning loop} at each iteration. A lightweight instruction-tuned MLLM serves as the reasoning controller. This planner composes executable steps such as \emph{SELECT}, \emph{ABSTAIN}, or \emph{COUNT}, based on both textual inputs and verified visual states. Each proposed step is immediately validated by perceptual faithfulness metrics before execution, ensuring that the reasoning process evolves only evidence-supported actions. To facilitate this, \textsc{FaithAct} provides an extensible interface composed of callable functions that serve as structured APIs. These functions supply multimodal evidence signals and can be invoked by the MLLM during planning, enabling the model to retrieve, verify, and reason over perceptual information in a unified and controlled manner.

\texttt{Poll()}: Returns the probability of objects' existence based on the polling model (Sec.~\ref{subsec:poll}).

\texttt{Ground()}: Returns the bounding boxes and confidence scores of a claimed object detected by \code{GroundingDINO} (Sec.~\ref{subsec:grounding}).

\texttt{Select()}: Selects an object as \emph{existent} if its confidence score exceeds the threshold in Sec.~\ref{subsec:judge}.

\texttt{Abstain()}: Abstains from selecting an object if its overall confidence is below the threshold.

\texttt{Count()}: Counts the number of \emph{reliably grounded} bounding boxes returned by \texttt{Ground()}, yielding the object count for quantitative reasoning.

Among these, \texttt{COUNT()} is a functional reasoning operation, while the remaining calls are verification functions that enforce perceptual faithfulness constraints.  
This modular design allows new functions, such as \texttt{Attribute()} or \texttt{Relate()}, to be incorporated, and supports further optimization or refinement strategies following \texttt{Abstain()}.

\begin{table*}[tph!]
\centering
\small
\resizebox{\linewidth}{!}{
\begin{tabular}{r|cccc|c}
\toprule
\textbf{Baseline Methods \& Datasets} & \textbf{\dataset{LLaVA-bench} (\%)} & \textbf{\dataset{RealWorldQA} (\%)} & \textbf{\dataset{POPE} (\%)} & \textbf{\dataset{MMHal} (\%)} & \textbf{Average (\%)} \\
\midrule
\rowcolor{promptgray}
Qwen + CoT & 46.05$\pm$19.58 & 48.11$\pm$27.04 & 45.21$\pm$24.87 & 53.34$\pm$24.02 & 48.18 \\
\rowcolor{promptgray}
\quad + VAT & 51.59$\pm$21.37 & 50.13$\pm$26.43 & 21.46$\pm$19.20 & 55.32$\pm$28.58 & 44.62 \\
\rowcolor{promptgray}
InternVL + CoT & 45.63$\pm$16.60 & 44.23$\pm$25.43 & 43.25$\pm$23.27 & 53.17$\pm$23.64 & 46.57 \\
\rowcolor{promptgray}
\quad + \textsc{VAT} & 48.97$\pm$17.22 & 45.31$\pm$28.19 & 40.26$\pm$22.47 & \underline{54.51$\pm$26.99} & 47.26 \\
\rowcolor{promptgray}
LLaVA + CoT & 47.56$\pm$23.35 & 52.31$\pm$28.44 & 52.28$\pm$25.66 & 30.63$\pm$28.56 & 45.70 \\
\rowcolor{promptgray}
\quad + VAT & 46.16$\pm$19.46 & 50.15$\pm$30.66 & \underline{52.59$\pm$27.15} & 30.30$\pm$28.96 & 44.80 \\
\midrule

\rowcolor{toolblue}
Qwen + Grounded-CoT & 50.04$\pm$17.54 & 53.35$\pm$26.68 & \underline{53.49$\pm$22.47} & \underline{56.77$\pm$25.86} & \underline{53.41} \\
\rowcolor{toolblue}
\quad + ReAct & \underline{54.82$\pm$26.53} & \underline{56.82$\pm$31.71} & 45.02$\pm$25.04 & 33.76$\pm$28.43 & 47.61 \\
\rowcolor{toolblue}
InternVL + Grounded-CoT & 48.35$\pm$18.05 & 47.94$\pm$19.36 & 17.44$\pm$19.01 & 18.10$\pm$14.96 & 32.96 \\
\rowcolor{toolblue}
\quad + ReAct & \underline{51.97$\pm$24.08} & \underline{56.56$\pm$31.30} & \underline{52.32$\pm$24.57} & 31.61$\pm$29.61 & \underline{48.11} \\
\rowcolor{toolblue}
LLaVA + Grounded-CoT & 50.62$\pm$18.74 & 52.30$\pm$28.89 & 50.56$\pm$25.86 & 31.69$\pm$27.00 & 46.29 \\
\rowcolor{toolblue}
\quad + ReAct & \textbf{59.20$\pm$27.18} & \underline{56.82$\pm$31.71} & 46.09$\pm$34.04 & \underline{32.23$\pm$31.73} & \underline{48.59} \\

\midrule

\rowcolor{faithgreen}
Qwen + FaithAct & \textbf{55.10$\pm$20.14} & \textbf{57.22$\pm$27.85} & \textbf{56.87$\pm$24.29} & \textbf{66.45$\pm$27.87} & \textbf{58.91} \\
\rowcolor{faithgreen}
InternVL + FaithAct & \textbf{52.64$\pm$17.75} & \textbf{57.35$\pm$29.40} & \textbf{56.01$\pm$21.76} & \textbf{61.71$\pm$27.01} & \textbf{56.93} \\
\rowcolor{faithgreen}
LLaVA + FaithAct & \underline{52.82$\pm$22.77} & \textbf{58.11$\pm$30.37} & \textbf{56.09$\pm$27.71} & \textbf{39.91$\pm$27.92} & \textbf{51.73} \\

\bottomrule
\end{tabular}
}
\caption{\textbf{Faithfulness evaluation across reasoning paradigms.}
We report the mean and standard deviation of the chain-level faithfulness score
$F_{\text{chain}}$ (in \%) on four benchmarks. Methods are organized by \emph{reasoning paradigm}: prompt-based reasoning (gray shading), tool-augmented reasoning (blue shading), and faithfulness-first planning (green shading), while holding backbone models fixed to enable controlled comparison. Best and second-best results within each backbone model are highlighted in \textbf{bold} and \underline{underline}, respectively.}
\label{tab:faith_eval_grouped}
\end{table*}

\subsection{Action-Guided Reasoning Refinement}
\label{subsec:refine}

The overall reasoning process of \textsc{FaithAct} is summarized in Algorithm~\ref{algo:fplan} in Appx.~\ref{appendix:algo}, which follows a \emph{refine-based} procedure. After verification, any reasoning step that fails to meet the perceptual faithfulness threshold is not discarded outright but revised through an \emph{action-guided refinement} process.  
In this phase, the planner re-invokes the MLLM with updated evidence from the evaluator, such as object existence labels, bounding boxes, and counts, to regenerate or adjust the unfaithful steps.  
The refinement prompt 
guides the model to produce a reasoning trace that remains consistent with perceptual evidence while maintaining logical continuity.  
This action-guided correction mechanism ensures that \textsc{FaithAct} preserves both interpretability and grounding, progressively improving the reasoning chain through verified feedback.
Further implementation and prompt details are provided in Appx.~\ref{appendix:f_plan_prompts}.

\section{Experiments}
We conduct an extensive experimental study demonstrating the effectiveness of Faithful-First RPA framework in measuring and improving faithfulness in MLLM reasoning. Unless otherwise noted, all results are reported as percentages (\%).

\begin{figure*}[t]
    \centering
    \includegraphics[width=0.93\linewidth]{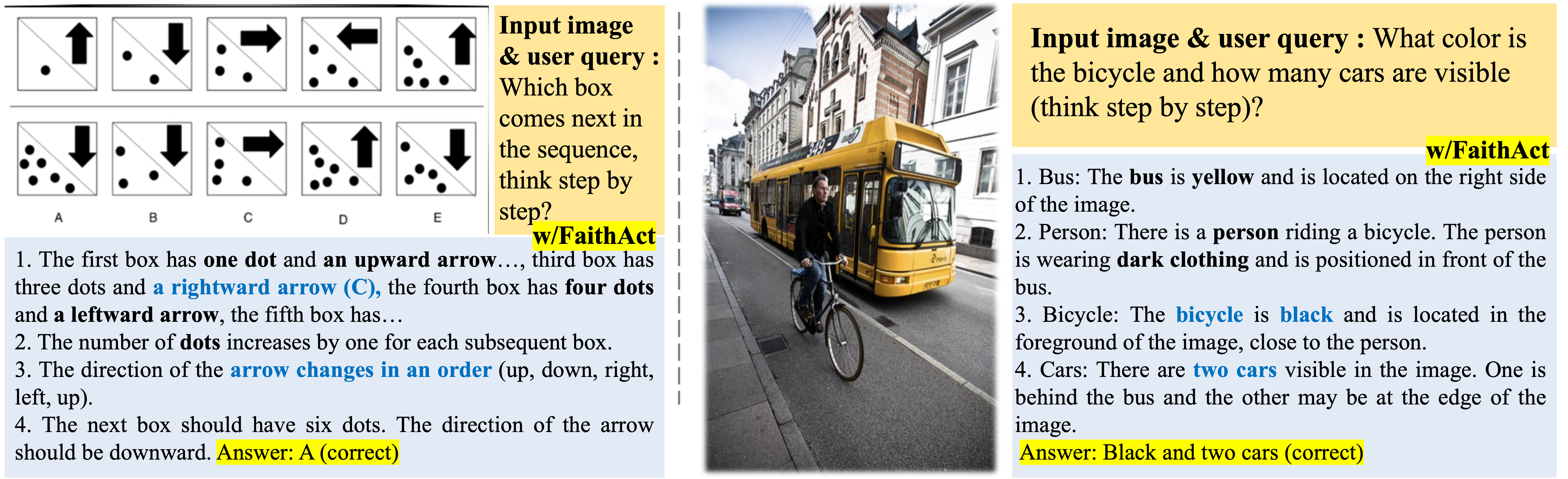}
    \caption{Qualitative comparison of reasoning chains generated with and without \textsc{FaithAct} on two illustrative cases. In both tasks, \textsc{FaithAct} enforces step-level perceptual verification, correcting hallucinated descriptions (red colored in Fig.~\ref{fig:motivation}) and producing more structured, visually grounded reasoning (colored in blue).}
    \label{fig:case_study}
\end{figure*}

\begin{figure}[htbp]
    \centering
    \includegraphics[width=\linewidth]{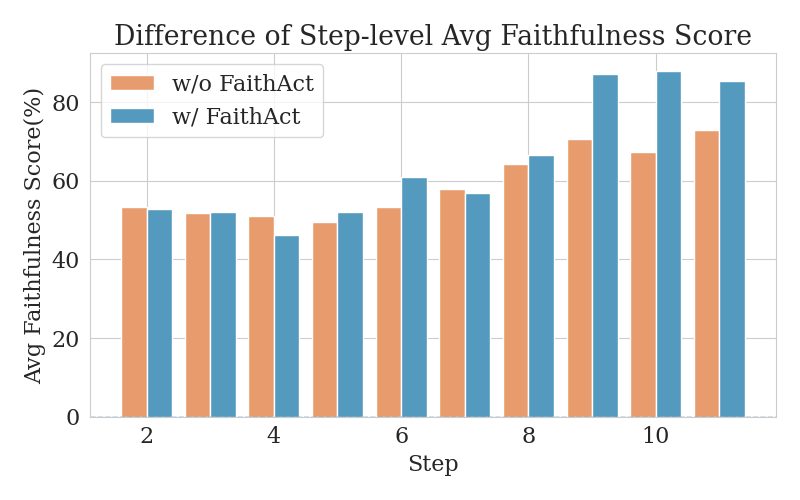}
    \caption{Distribution of average $F_{\text{step}}$ difference across reasoning steps. The x-axis are reasoning steps and y-axis represents the $F_{\text{step}}$ averaged difference between \model{Qwen} with and without \textsc{FaithAct}.}
    \label{fig:distribution}
    \vskip -0.05in
\end{figure}

\subsection{Experimental Setup}

\noindent \textbf{Datasets.} 
We evaluate our framework on widely used multimodal benchmarks covering object recognition, visual grounding, and hallucination-sensitive question answering, including \dataset{LLaVA-bench}~\cite{liu2023visual}, \dataset{RealWorldQA}~\cite{xai_grok1_5}, \dataset{POPE}~\cite{li2023evaluating}, and \dataset{MMHal-Bench}~\cite{2023llavarlhf} (MMHal).
These datasets feature images with rich real-world objects and are well suited for evaluating perceptual faithfulness. Specifically, the queried objects in POPE are removed to ensure fair assessment.

\noindent \textbf{Baselines.} 
We compare our framework with representative training-free reasoning frameworks, including \textsc{CoT}~\cite{wei2022chain, zhang2023multimodal}, \textsc{ReAct}~\cite{yao2022react}, \textsc{Grounded-CoT}~\cite{Wu2025GroundedCF}, and \textsc{Visual Abstract Thinking} (VAT)~\cite{liu2025visual}. 
These methods are evaluated on widely used MLLMs, namely \code{Qwen-2.5-VL-7B}(Qwen)~\cite{wang2024qwen2}, \code{InternVL3-8B}(InternVL)~\cite{zhu2025internvl3}, and \code{LLaVA-one-Vision-1.5-8B}(LLaVA)~\cite{an2025llava}.

\noindent \textbf{Hyper-parameters setup.} The box threshold of \code{GroundingDINO} is 0.35 and the text threshold is 0.25. We keep the default generation settings for helper LLM and MLLMs.

\subsection{Experimental Results and Analysis}
Table~\ref{tab:faith_eval_grouped} shows the performance of \textsc{FaithAct} compared with selected baseline methods. The experimental results highlight several key findings. 

\paragraph{\textbf{Perceptual faithfulness is broadly underestimated.}}
Across all evaluated models, faithfulness remains far from ideal. For instance, MLLMs such as \code{Qwen} achieve faithfulness scores of only around $\sim$50 across datasets and evaluation metrics, substantially below the desired level (near 100). This highlights that current reasoning models often generate partially ungrounded or inconsistent reasoning traces, indicating that perceptual verification and evidential grounding are still open challenges.

\paragraph{\textsc{FaithAct} improves faithfulness across models.}
Introducing the \textsc{FaithAct} framework consistently enhances reasoning faithfulness across three tested models. For example, \code{InternVL} achieves 57.35~$\pm$~29.40 on \dataset{RealWorldQA} with \textsc{FaithAct}, compared to 44.23~$\pm$~25.43 without it. Generally, our method attains the \textbf{highest} faithfulness in \textbf{11 out of 12} evaluated settings, demonstrating its effectiveness across different architectures and datasets. Averaged across models, \textsc{FaithAct} achieves a mean score of 55.86\%, outperforming the strongest baseline reasoning paradigm (ReAct, 48.10\%) by 7.76 percentage points and \textsc{CoT} by 9.04\%.
We show that $F_\text{chain}$ of \textsc{ReAct} is theoretically bounded by that of \textsc{FaithAct} (in Appx.~\ref{appendix:theory}). We also note that the reported standard deviations are relatively large, which is expected given the instance-level nature of perceptual faithfulness.

\paragraph{\textsc{FaithAct} mitigates hallucination.} Although \textsc{FaithAct} is not explicitly designed to mitigate hallucinations, it yields substantial gains on hallucination-focused benchmarks.
In particular, on \textsc{MMHal}, 
models with \textsc{FaithAct} exhibit a marked reduction in hallucinated reasoning steps, achieving an average improvement of 21.99\% and 9.81\% over tool-augmented prompt-based reasoning frameworks, respectively.
For example, \code{Qwen} with \textsc{FaithAct} achieves 66.45~$\pm$~27.87, surpassing its second-best score. This notable gain indicates that the principle enforced by \textsc{FaithAct} effectively constrains perceptual unfaithfulness during intermediate reasoning steps, thereby enhancing faithfulness of model outputs.


\paragraph{Improving faithfulness does not degrade task performance.}
Table~\ref{tab:correctness} reports the performance of three models with and without \textsc{FaithAct} across the two benchmarks. The results show that integrating \textsc{FaithAct} preserves the model's ability to generate correct final answers to multimodal questions, with slight improvements observed in two out of three datasets. These findings indicate that \textsc{FaithAct} enhances the faithfulness of reasoning steps without harming the model's original performance. Moreover, although behavioral faithfulness is not explicitly measured, the empirical evidence supports our hypothesis that perceptually grounded reasoning encourages behavioral consistency.


\paragraph{Faithfulness increases across reasoning steps.} Fig.~\ref{fig:distribution} reports the distribution of $F_{step}$, when comparing \textsc{FaithAct} and raw \textsc{CoT} without \textsc{FaithAct}. We observe that the benefit of \textsc{FaithAct} becomes particularly pronounced in the later reasoning steps, suggesting that its intervention is most effective when the model engages in deeper chains of reasoning. This observation is consistent with prior work~\cite{wu2025more} which points out that excessively \textsc{CoT} increases susceptibility to noise in their later steps and thereby leads to more unfaithfulness.

\subsubsection{Qualitative Analysis and Case Studies.} We conduct a qualitative comparison to illustrate how \textsc{FaithAct} improves reasoning faithfulness in the motivating cases shown in Fig.\ref{fig:motivation}. In the first example (Fig.~\ref{fig:case_study} top), the baseline model hallucinates \textit{a yellow bicycle} and \textit{no cars} by relying on language priors or implication (e.g., yellow bus) rather than image evidence. With \textsc{FaithAct}, each reasoning step is perceptually verified, leading to the correct identification of a \textit{black bicycle} and \textit{two cars}. In the second example (bottom), both models correctly predict the next box in a visual sequence, but with \textsc{FaithAct} guidance, reasoning chain is more structured and explicitly justifies each visual attribute transition. For example, the direction of upward arrow in (C) is corrected to rightward. These cases demonstrate that \textsc{FaithAct} enhances perceptual grounding and visually consistent reasoning processes, empirically supporting and improving behavioral consistency.


\subsection{Human Validation on Extracted Objects}
To evaluate whether \code{Qwen} extracted objects (Sec.~\ref{subsec:objectext}) accurately capture what is indeed stated in the text, we conducted a human validation study using 50 snippet and 25 annotators, producing a total of $7,550$ object-level labels, detailed in Appx. (~\ref{app:human_validation}). 
Compared with human judgments, the LLM shows near-perfect consistency, achieving a precision of \textbf{99.42\%} with only $44$ false positives (over-extraction rate = 0.58\%).
At snippet level, we compute \emph{snippet validity}, defined as the probability that all extracted objects for a snippet are judged correct by a human annotator. Among $1,250$ annotator--snippet pairs, the LLM attains a mean snippet validity of \textbf{0.9680}, showing extracted object sets are almost fully aligned with human.

\begin{table}[t]
    \centering
    \resizebox{\linewidth}{!}{
    \begin{tabular}{l|cc}
    \toprule
       \textbf{Datasets \& Models}   & \textbf{\dataset{RealWorldQA} (\%)} & \textbf{\dataset{MMHal} (rating)} \\
       \midrule
        Qwen + CoT & 70.1 & 3.40 \\
        
        +\textbf{FaithAct (ours)} &  74.5 & 3.48 \\
        \midrule
       InternVL + CoT & 70.8 & 3.61 \\
        
        +\textbf{FaithAct (ours)} & 71.2 & 3.58 \\
        \midrule
       LLaVA + CoT & 68.1 & 3.41 \\
        
        +\textbf{FaithAct (ours)} & 67.8 & 3.46 \\
        \bottomrule
    \end{tabular}
    }
    \caption{\label{tab:correctness}
    Task performance with \textsc{CoT} and  \textsc{FaithAct}.}
    \vskip -0.05in
\end{table}

\subsection{Ablation and Sensitivity Analysis}
We conduct ablation and sensitivity analyses to examine the contribution of individual components in the proposed framework. In particular, we focus on the two core functions \texttt{Poll()} and \texttt{Ground()}. Throughout these experiments, the faithfulness evaluation protocol remains unchanged and we disable one module at a time to isolate its effect.

Results on \code{Qwen} are reported in Table~\ref{tab:ablation}. Removing either \texttt{Poll()} or \texttt{Ground()} leads to a noticeable decrease in faithfulness (approximately 5\%), with \texttt{Ground()} having a slightly larger impact. This suggests that object localization provides essential visual evidence for perceptually grounded reasoning. The results indicate that the two modules are complementary, and that \textsc{FaithAct} achieves the greatest improvement in faithfulness when both are jointly enabled.

We additionally conduct a study replacing \code{GroundingDINO} with a recently published alternative, \code{SAM3}~\cite{carion2025sam}.
Equiped with \code{SAM3}, the performance of \textsc{FaithAct} drops obviously (5\% on \dataset{RealWorldQA} and more on \dataset{MMHal}). This may suggest that \textsc{FaithAct} needs localization-specific models (like \code{GroundingDINO}) to provide grounding results more accurate with proper confidence.

we also choose Qwen on RealWorldQA for an additional ablation analysis over various threshold values. Results are reported in Table~\ref{tab:sensitive_param}. Notably, the changes in hyperparameters do not influence the final performance of FaithAct too much. The threshold 0.5-0.5 will lead to a slight decline in performance.

\begin{figure}[t]
    \centering
    \includegraphics[width=.95\linewidth]{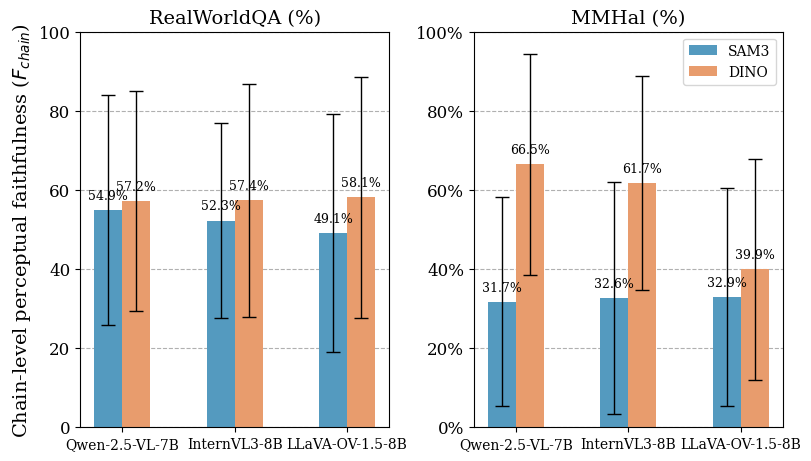}
    \caption{Comparative performance of the \textsc{FaithAct} framework using \code{SAM3} vs. \code{GroundingDINO} as the \texttt{Ground()} function. Results show mean accuracy ($\%$) and standard deviation (error bars) across three MLLMs on the \dataset{RealWorldQA} and \dataset{MMHal} datasets.}
    \label{fig:sam3}
    \vskip -0.05in
\end{figure}

\begin{table}[t]
    \centering
    \resizebox{\linewidth}{!}{
    \begin{tabular}{l|ll}
    \toprule
       \textbf{Datasets \& Models}   & \textbf{\dataset{RealWorldQA} (\%)} & \textbf{\dataset{MMHal} (\%)} \\
       \midrule
        FaithAct & \textbf{57.22$\pm$27.85} & \textbf{66.45$\pm$27.87} \\
        
        FaithAct (w/o Poll) & 54.24$\pm$28.13 & 63.25$\pm$26.75 \\

        FaithAct (w/o Ground) & 53.16$\pm$29.12  & 62.47$\pm$28.83 \\
        
        \bottomrule
    \end{tabular}
    }
    \caption{Ablation study of the two core components.}
    \label{tab:ablation}
    \vskip -0.10in
\end{table}

\begin{table}[t]
\centering
\resizebox{\linewidth}{!}{
\begin{tabular}{l|ccc}
\hline
\textbf{$\alpha$ \& threshold} & 0.3-0.7 & 0.4-0.6 & 0.5-0.5 \\
\hline
0.5 & 55.76$\pm$27.55 &	54.15$\pm$26.09 &	52.52$\pm$26.63 \\
0.6 &   55.84$\pm$27.59	& 55.96$\pm$23.73	& 52.58$\pm$24.62 \\
0.7 & 57.44$\pm$26.24 & 57.22$\pm$27.85	& 52.72$\pm$21.47
 \\
\hline
\end{tabular}
}
\caption{Sensitivity analysis on $\alpha$ and threshold.}
\label{tab:sensitive_param}
\end{table}

\subsection{Computational Overhead}
\label{sec:computational}
We report the average time per chain generation (in seconds) using Qwen-2.5-VL-7B under identical decoding settings in Table~\ref{tab:time_comparison}. Importantly, our method does not modify the underlying generation function or decoding procedure of the MLLM. Consequently, the token generation rate (tokens/second) remains effectively the same as that of raw CoT. We also report reasoning-chain lengths in Appendix~\ref{appendix:chain_length}.

\subsection{Perception-light Benchmarks}
\label{sec:perception_light}
we have additionally evaluated FaithAct on MathVista~\cite{lumathvista} and ScienceQA~\cite{lu2022learn}, where visual evidence is necessary but not sufficient and abstract or external knowledge is often required. Table~\ref{tab:perception_light} reports the  $F_{\text{chain}}$ of FaithAct and raw CoT (\%) on these datasets. FaithAct performs well on ScienceQA and slightly better than CoT on MathVista.

\section{Discussion: Behavioral Faithfulness}
\label{sec:discussion}
To provide more analysis and insights between behavioral faithfulness and perceptual faithfulness, we added a qualitative taxonomy to showing that perceptual faithfulness (PF) and behavioral faithfulness (BF) can dissociate across all four regimes (BF+ PF+; BF+ PF-; BF- PF+; BF- PF-, '+' and '-' mean faithful and unfaithful). Importantly, PF+ does not guarantee BF+, where a model may correctly ground intermediate observations but still produce an incorrect final decision due to downstream reasoning errors. Conversely, BF+ does not imply PF+, where a model can arrive at the correct answer while relying on perception-unfaithful rationales. We will also clarify a standard operationalization of behavioral faithfulness from prior works in the main text. For example, if the model truly uses the cited evidence, then removing or perturbing that evidence should systematically alter the model's decision.

\section{Conclusion}
We introduced Faithful-first RPA, a framework that enforces perceptual grounding throughout the reasoning process. Within this framework, \textsc{FaithEvi} provides a principled and fine-grained evaluation of perceptual faithfulness, while \textsc{FaithAct} operationalizes faithfulness-first reasoning through planning and acting. 
Experiments across benchmarks show that our framework improves perceptual faithfulness by up to 24\% without compromising accuracy, effectively mitigating hallucinations. These results underscore the value of faithfulness as a core design principle. Future work will extend to behavioral faithfulness and more challenging open-ended reasoning settings.
\begin{table}[t]
\centering
\setlength{\tabcolsep}{2pt}
\resizebox{\linewidth}{!}{
\begin{tabular}{l|cccc}
\hline
\textbf{Datasets} & \textbf{\dataset{LLaVA-BENCH}} & \textbf{\dataset{REALWORLDQA}} & \textbf{\dataset{MMHAL}} & \textbf{\dataset{POPE}} \\
\hline
FaithAct (s) & 18.85 & 14.04 & 16.48 & 14.72 \\
CoT (s)      & 10.73 & 4.65  & 5.59  & 3.41  \\
$\delta$ Time & 8.12  & 9.39  & 10.89 & 11.31 \\
\hline
\end{tabular}
}
\caption{Inference time comparison.}
\label{tab:time_comparison}
\end{table}

\begin{table}[t]
\centering
\resizebox{\linewidth}{!}{
\begin{tabular}{l|cc}
\hline
\textbf{Dataset \& Methods} & \textbf{\dataset{MathVista}} & \textbf{\dataset{ScienceQA}} \\
\hline
FaithAct &	47.65$\pm$24.73	& 52.97$\pm$32.63 \\
CoT	& 47.00$\pm$23.91	& 40.14$\pm$32.88
 \\
\hline
\end{tabular}
}
\caption{Faithfulness evaluation on perception-light benchmarks.}
\label{tab:perception_light}
\vskip -0.15in
\end{table}


\section*{Limitations}
This work primarily focuses on perceptual faithfulness and does not directly evaluate behavioral faithfulness, i.e., the alignment between reasoning traces and the model’s final decision process. While our empirical results suggest that enforcing perceptual grounding may be associated with more behaviorally consistent outputs, this relationship is not explicitly measured or guaranteed.

In addition, although we conduct a human validation study indicating that LLM-extracted objects are generally accurate, we do not perform large-scale human evaluations of step-level or chain-level perceptual faithfulness. Complementary human studies at different levels of granularity could help further contextualize and validate our findings.

Finally, our current implementation verifies perceptual faithfulness primarily at the level of object existence. Extending the Faithful-first RPA framework to incorporate attribute- and relation-level verification remains an important direction for future work, and may further reduce perceptual unfaithfulness in cases where objects are present but their properties or relations are mischaracterized.

\section*{Acknowledgements}
This work was supported in part by computing resources provided by Lambda AI.

\bibliography{custom}

\clearpage

\appendix
\section{Related Work}
\label{sec:related_frameworks}

\noindent \textbf{Reasoning Frameworks for MLLMs.}
MLLMs have developed surprising reasoning capabilities in multimodal problem solving, task planning, scientific discovery and so on~\cite{lumathvista, li2025chemvlm, gao2024cantor,an2024mc, an2025unictokens, lin2025perceive, li2025miv, li2025taco}. 
To enhance systematic reasoning, several frameworks~\cite{chen2024m3cot, sun2025mm, shao2024visual, zhou2025logic, wang2025vl} have been proposed to decompose multimodal problems into interpretable steps. Among them, \textsc{CoT}~\cite{chen2024m3cot}, \textsc{Grounded-CoT} \cite{Wu2025GroundedCF} and \textsc{ReAct} frameworks~\cite{yao2022react} are famous and insightful ones. 
More recently, \textsc{VAT}~\cite{liu2025visual} explores hierarchical and compositional reasoning, encouraging models to form abstract visual concepts that support complex decision-making. Such generated reasoning chains, though often fluent and logically structured, may still include steps unsupported by visual evidence or inconsistent with the model's actual decision process~\cite{yu2024rlhf}. 

\noindent \textbf{Faithfulness in Multimodal Reasoning. } 
Despite the demonstrated effectiveness of reasoning in improving the performance of MLLMs, an unfaithful problem emerges: the reasoning traces generated by these models do \textbf{not always behave faithfully} to the internal processes that produce their final answers \cite{lyu2023faithful,turpin2023language,lanham2023measuring}.
Models often rely on latent knowledge or shortcut associations that are not explicitly expressed in their reasoning chains. 
As a result, the generated reasoning steps may read as a plausible but \textbf{untrustworthy explanation}, motivating the need for explicit faithfulness assessment \cite{jacovi2020towards,parcalabescu2023measuring,matton2025walk,barez2025chain,arcuschin2025chain,peng2025logic,peng2025stepwise}. 
Several works have introduced benchmarks or metrics to evaluate faithfulness explicitly. M$^4$~\cite{li2023m4}, \textsc{FaithScore}~\cite{jing2023faithscore} and \textsc{TIFA}~\cite{hu2023tifa} propose metrics for evaluating the faithfulness for vision-language models. However, we find that behavioral alignment (behavioral faithfully) does not guarantee the correctness of the final output (see \textit{left panel} in Fig.\ref{fig:motivation}). 

\noindent \textbf{Object Hallucination as Unfaithful Consequences.} Object hallucination is identified as a common challenge, as a consequence, in Large Vision-Language Models (LVLMs), where the model describes or reasons about objects absent in the input image. Several benchmarks, such as \dataset{POPE}~\cite{li2023evaluating} and \dataset{MMHal-Bench}~\cite{sun2024aligning}, have been developed to systematically evaluate this phenomenon. However, these efforts typically treat hallucination as an isolated failure mode rather than as a manifestation of broader \textbf{unfaithful reasoning}. 

Recent studies address hallucination through different mechanisms: (i) \emph{training-time alignment}, such as hallucination-aware \textsc{Reinforcement Learning from Human Feedback} or preference optimization~\cite{sun2024aligning,zhang2024hallucination}; (ii) \emph{decoding-time constraints}, including grounded or contrastive decoding~\cite{li2023evaluating}; and (iii) \emph{feature-level grounding strategies} that enhance cross-modal alignment~\cite{ghosh2024vdgd}.
Together, these advances underscore that faithful reasoning requires not only linguistic coherence but also perceptual accountability, further motivating to verify evidential grounding before inference.

\section{Object Extraction Prompts and Examples}
\label{appendix:object_ext_prompt}
\begin{tcolorbox}[notitle, boxrule=0pt,left=0.05cm, right=0.05cm, top=0cm, bottom=0cm]
\textit{Extract all objects mentioned in the following sentence that may occur in an image. Only extract nouns meaning objects, not abstract adjectives, concepts, actions, general nouns or locations. Do not include non-object nouns or words like ``Image'', ``Object'', ``Feature'', or ``Photo''.  \textbackslash n \textbackslash n\#\#\#\{\textcolor{blue}{One Reasoning Step}\}\#\#\#  \textbackslash n  \textbackslash nReturn only a list of nouns like [``xxx'', ``xxx'', ``xxx''] and do not include any other things. If no available nouns, return an empty list [].}
\end{tcolorbox} 

\noindent \textbf{Example 1}:

\noindent Text: \textit{**Location Context**: The presence of a coastal area with a beach and a city in the background suggests a location near the ocean.}

\noindent Extract Result: \textit{[``coastal area'', ``beach'', ``city'']}

\noindent \textbf{Example 2}: 

\noindent Text: \textit{**Setting**: The image appears to be taken on a city street, likely in an urban area given the presence of taxis and buildings in the background.}

\noindent Extract Result: \textit{[``taxis'', ``buildings'']}

\section{Training Details on POPE}
\label{appendix:clip_training}
The training set of \dataset{POPE} released on Huggingface is \href{https://huggingface.co/datasets/lmms-lab/POPE}{here}. It consists of three parts, random, popular and adversarial, all with image-object existence labels. The total training size is 9000 items. We utilize the released test set for testing model performance.

The CLIP+polling head model is trained on two NVIDIA RTX4090 48GB GPUs. During training, we freeze the backbone \textsc{CLIP}, and only set the head trainable. We set the batch size to $32$ each GPU, learning rate to $1e-3$, and train $50$ epochs with early stop. The final test accuracy of the model is $99.80\%$, with $9000$ real-world test examples in POPE. We test the trained model in the wild, too. And we discover that it can reliably tell whether an object exists in an image. Thus, we can safely use it in our preference polling task.


\section{Prompts for \textsc{FaithAct}}
\label{appendix:f_plan_prompts}
\begin{tcolorbox}[notitle, boxrule=0pt,left=0.05cm, right=0.05cm, top=0cm, bottom=0cm]
\textit{Question: \{your original question\}.\textbackslash n \textbackslash n Model Response: \{MLLM's original response\}\textbackslash n \textbackslash n Additional location information:\textbackslash n \textbackslash n \{Information from the functions\}\textbackslash n \textbackslash n Using only the ``exists'' objects with high confidence and avoid using objects that do not exist. Do not include new objects or descriptions. Do not repeat the evidences, confidence scores and bounding boxes in your reasoning. Think step by step. Steps should be like: 1.<object1>:xxx\textbackslash n \textbackslash n  2.<object2>:xxx\textbackslash n \textbackslash n ...\textbackslash n \textbackslash n ..., .}
\end{tcolorbox} 

\section{Algorithm of \textsc{FaithAct}}
\label{appendix:algo}
Here we list the algorithm process of \textsc{FaithAct} in Algorithm~\ref{algo:fplan}.
\begin{algorithm}[t]
\caption{Faithfulness-First Planner (\textsc{FaithAct})}
\label{algo:fplan}
\begin{algorithmic}[1]
\REQUIRE Image $I$, textual query $Q$, MLLM planner $M$, helper LLM $H_M$
\ENSURE Faithful reasoning output $R$

\STATE \textbf{// Step 1: Initial Reasoning}
\STATE $R_{\text{raw}} \gets M.\text{Reason}(I, Q)$

\STATE \textbf{// Step 2: Extract Claimed Objects}
\STATE $O_{\text{raw}} \gets H_M.\text{ExtractObj}(R_{\text{raw}})$

\STATE \textbf{// Step 3: Verification and Function Calls}
\FOR{each object $o_i \in O_{\text{raw}}$}
    \STATE $c_{p}^i \gets \texttt{Poll}(o_i)$ \COMMENT{Existence confidence}
    \STATE $B_i, c_{g}^i \gets \texttt{Ground}(o_i)$ \COMMENT{Bounding boxes and spatial scores}
    \STATE $f^i \gets \texttt{Select}(c_{p}^i, c_{g}^i)$ or $\texttt{Abstain}(c_{p}^i, c_{g}^i)$ \COMMENT{Faithfulness threshold check}
    \IF{\texttt{Select}}
        \STATE $n_i \gets \texttt{Count}(B_i, c_{g}^i)$ \COMMENT{Count reliably grounded instances}
        \STATE \textbf{record} $(o_i, f^i, n_i, B_i)$
    \ENDIF
\ENDFOR

\STATE \textbf{// Step 4: Refine-Based Faithful Reasoning}
\STATE $R_{\text{new}} \gets M.\text{Reason}(I, Q, O_{\text{raw}}, \{f^i, B_i, n_i\})$

\STATE \textbf{return} $R_{\text{new}}$
\end{algorithmic}
\end{algorithm}

\input{sec/theory}

\section{Human Validation of LLM-Extracted Objects}
\label{app:human_validation}

This section provides detailed methodology and analysis for the human validation study used to assess the accuracy of LLM-extracted objects described in Sec. \ref{subsec:objectext}.

\subsection{Study Design and Data Collection}

We randomly sampled 50 text snippets from our evaluation corpus. Each snippet was processed by the \code{Qwen} to extract a set of candidate objects that were intended to represent entities explicitly or implicitly stated in the text. To validate the correctness of these extracted objects, we conducted a human annotation study involving 25 annotators.

Each annotator independently evaluated all extracted objects for each snippet. For every object, annotators answered the binary question:
\emph{``whether the extracted object explicitly exists in the text.''}
Objects were labeled as \emph{supported} (1) if they were explicitly mentioned or unambiguously implied by the text, and as \emph{unsupported} (0) otherwise. Annotators were instructed to rely solely on the provided text and to avoid using external world knowledge. When uncertain, they were instructed to mark the object as unsupported.

In total, the study yielded 7,550 object-level labels, corresponding to 1,250 annotator--snippet pairs (50 snippets $\times$ 25 annotators).

\begin{table*}[t]
\centering
\begin{tabular}{l|cccc}
\hline
\textbf{Dataset \& Metrics} & \textbf{\dataset{LLaVA-bench}} & \textbf{\dataset{RealWorldQA}} & \textbf{\dataset{MMHal}} & \textbf{\dataset{POPE}} \\
\hline
Avg token length (FaithAct) & 270.88 & 171.43 & 112.85 & 149.11 \\
Avg token length (CoT)      & 357.52 & 162.78 & 200.06 & 148.23 \\
Min token length (FaithAct) & 30     & 17     & 12     & 28     \\
Min token length (CoT)      & 98     & 64     & 45     & 43     \\
Max token length (FaithAct) & 512    & 512    & 510    & 512    \\
Max token length (CoT)      & 512    & 432    & 502    & 512    \\
\hline
\end{tabular}
\caption{Token length statistics for FaithAct and CoT across different benchmarks.}
\label{tab:token_length_stats}
\end{table*}

\subsection{Evaluation Metrics}

We first evaluate precision at the object level by comparing LLM-extracted objects against human judgments. An extracted object is considered a false positive if it is labeled as unsupported by a human annotator. We compute object-level precision as
\begin{equation*}
\text{Precision} = \frac{\text{\# supported objects}}{\text{\# extracted objects}}.
\end{equation*}
Across all object-level annotations, the LLM achieves a precision of 99.42\%, with only 44 false positives, corresponding to an over-extraction rate of 0.58\%. 

While precision metrics capture local correctness, they do not reflect whether an entire set of extracted objects for a snippet is jointly accurate. To assess holistic correctness, we introduce \emph{snippet validity}, defined as an indicator function over annotator--snippet pairs:
\begin{equation*}
\small
\text{SV}(a, s) =
\mathbf{1}\big[\forall o \in \mathcal{O}_{s},\; o \text{ is supported by annotator } a\big],
\end{equation*}
where $\mathcal{O}_{s}$ denotes the set of objects extracted for snippet $s$.

We report mean snippet validity by averaging $\text{SV}(a, s)$ across all annotator--snippet pairs:
\begin{equation*}
\text{Mean SV} = \frac{1}{|\mathcal{A}||\mathcal{S}|}
\sum_{a \in \mathcal{A}} \sum_{s \in \mathcal{S}} \text{SV}(a, s).
\end{equation*}
Across 1,250 annotator--snippet pairs, the LLM achieves a mean snippet validity of 0.97. This indicates that for nearly all annotator--snippet evaluations, \emph{all} extracted objects for a snippet are judged correct.

\paragraph{Summary}
The combined object-level and snippet-level analyses provide complementary views of the choice of Qwen as an object extractor. High object-level precision and high snippet validity demonstrates that extracted object sets are almost always entirely accurate, serving as reliable inputs for downstream faithfulness evaluation in our framework.

\begin{figure*}[t]
    \centering
    \includegraphics[width=\linewidth]{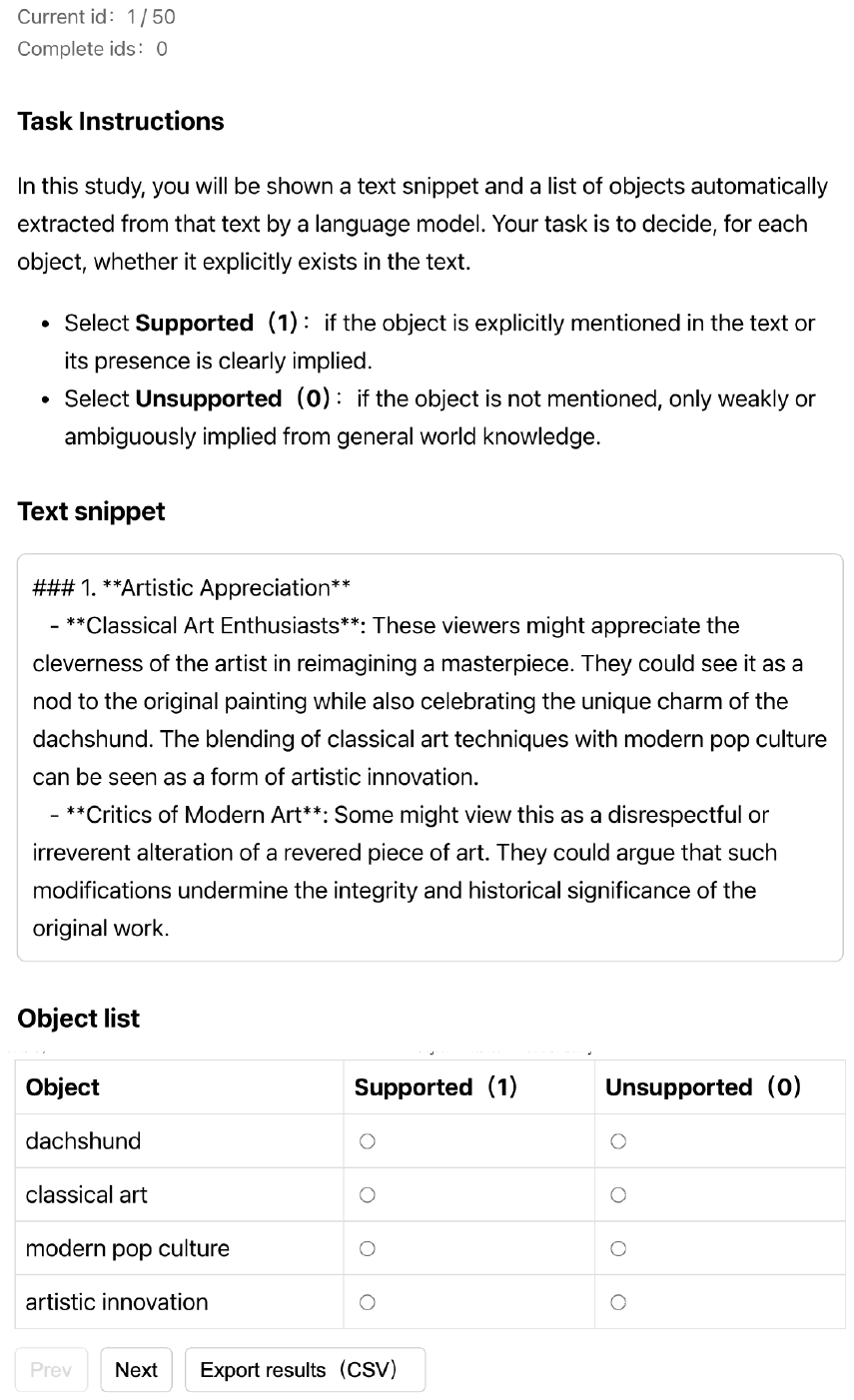}
    \caption{\textbf{Human annotation interface for object-existence validation.} For each snippet, annotators are presented with the original text and a set of objects automatically extracted by the LLM. Annotators judge whether each object is explicitly mentioned or unambiguously implied by the text, producing binary Supported (1) or Unsupported (0) labels used for evaluating extraction faithfulness.}
    \label{fig:human_study_case}
\end{figure*}

\section{Reasoning-Chain Length}
\label{appendix:chain_length}
We report the reasoning chain length of our method of Qwen-2.5-VL-7B in Table~\ref{tab:token_length_stats}, compared with raw CoT. The official tokenizer of the model is used. Here we report the metrics calculated from the final output of FaithAct. We find that FaithAct does not produce longer reasoning chains than raw CoT. In fact, on certain metrics, it yields shorter chains. A plausible explanation is that FaithAct constrains the model to reason only over visually verified objects with acceptable faithfulness scores. By filtering out unsupported or hallucinated objects, the framework thereby leads to more concise and focused chains.

\end{document}

%% file: sec/theory.tex
\section{Proof of faithfulness by FaithAct v.s. ReAct}
\label{appendix:theory}
\begin{lemma}[Faithfulness Dominance of FaithAct over ReAct]
Let $R^{\text{ReAct}} = \{ s_t^{\text{ReAct}} \}_{t=1}^{T}$ and $R^{\text{FaithAct}} = \{ s_t^{\text{FaithAct}} \}_{t=1}^{T'}$ denote the reasoning chains generated by ReAct and FaithAct, respectively. 
Let $F_{\text{step}}(s_t)$ be the perceptual faithfulness of step $s_t$, and define the chain-level faithfulness as
\begin{equation}
    F_{\text{chain}}(R) = \frac{1}{|R|}\sum_{t=1}^{|R|} F_{\text{step}}(s_t).
\end{equation}
Assume that FaithAct refines each candidate step $s_t^{(k)}$ using verified evidence such that 
\begin{equation}
    F_{\text{step}}\big(s_t^{(k+1)}\big) \ge F_{\text{step}}\big(s_t^{(k)}\big),
    \label{eq:monotone_refine}
\end{equation}
and accepts only refined steps satisfying $F_{\text{step}}(s_t) \ge c$ for some threshold $c \in [0,1]$. Then
\begin{equation}
    F_{\text{chain}}\!\left(R^{\text{FaithAct}}\right) \ge F_{\text{chain}}\!\left(R^{\text{ReAct}}\right).
\end{equation}
\end{lemma}

\begin{proof}
For each semantic subgoal $g$, let $s_g^{\text{ReAct}}$ denote the step generated by ReAct and $s_g^{(0)}$ the initial unverified step proposed by the same MLLM within FaithAct. 
FaithAct refines $s_g^{(0)}$ through iterative verification:
\[
s_g^{(0)} \rightarrow s_g^{(1)} \rightarrow \cdots \rightarrow s_g^{(K_g)} = s_g^{\text{FaithAct}}.
\]
By monotonicity in~\eqref{eq:monotone_refine},
\begin{equation}
    F_{\text{step}}\!\left(s_g^{\text{FaithAct}}\right)
    \ge
    F_{\text{step}}\!\left(s_g^{(0)}\right)
    =
    F_{\text{step}}\!\left(s_g^{\text{ReAct}}\right),
    \forall g.
    \label{eq:stepwise_dominance}
\end{equation}

If FaithAct drops an unverified claim (via \texttt{Abstain()}), it effectively removes a low-faithfulness step, which cannot decrease the average of the remaining step scores. 
Let $A$ and $B$ denote the multisets of step scores in FaithAct and ReAct, respectively. 
Then every element in $A$ dominates or replaces an element in $B$ with greater or equal score. 
Removing low-valued elements weakly increases the mean, hence
\[
\frac{1}{|A|}\sum_{a\in A} a \ge \frac{1}{|B|}\sum_{b\in B} b.
\]
By definition, this is equivalent to
\[
F_{\text{chain}}\!\left(R^{\text{FaithAct}}\right)
\ge
F_{\text{chain}}\!\left(R^{\text{ReAct}}\right).
\]
\end{proof}

\noindent
This result follows directly from FaithAct’s \emph{verify-and-refine} constraint: 
each reasoning step is either (i) retained and refined until it is perceptually grounded, or (ii) rejected through \texttt{Abstain()} if unsupported, ensuring that no unverified or hallucinated step reduces overall faithfulness.

\begin{corollary}[Strict Improvement Under Unfaithful Steps]
Under the assumptions of Lemma 1, suppose there exists at least one subgoal $g^\star$ such that the ReAct step $s_{g^\star}^{\text{ReAct}}$ is perceptually unfaithful, i.e.,
\begin{equation}
    F_{\text{step}}\!\left(s_{g^\star}^{\text{ReAct}}\right) < 1.
\end{equation}
Assume further that FaithAct either (i) refines this step into a perceptually grounded step $s_{g^\star}^{\text{FaithAct}}$ with
\begin{equation}
    F_{\text{step}}\!\left(s_{g^\star}^{\text{FaithAct}}\right) 
    >
    F_{\text{step}}\!\left(s_{g^\star}^{\text{ReAct}}\right),
\end{equation}
or (ii) rejects the claim via \texttt{Abstain()}, thereby removing $s_{g^\star}^{\text{ReAct}}$ entirely from its chain.
Then
\begin{equation}
    F_{\text{chain}}\!\left(R^{\text{FaithAct}}\right)
    >
    F_{\text{chain}}\!\left(R^{\text{ReAct}}\right).
\end{equation}
\end{corollary}

\begin{proof}
Case (i): If FaithAct refines $s_{g^\star}^{\text{ReAct}}$ into $s_{g^\star}^{\text{FaithAct}}$ with strictly higher step-level faithfulness, then by the stepwise dominance in~\eqref{eq:stepwise_dominance}, FaithAct strictly improves at least one term in the average while leaving all other terms weakly improved. The mean of a set of real numbers strictly increases if at least one element increases and no element decreases. Hence $F_{\text{chain}}(R^{\text{FaithAct}}) > F_{\text{chain}}(R^{\text{ReAct}})$.

Case (ii): If FaithAct abstains on $g^\star$, then ReAct includes a low-faithfulness step $s_{g^\star}^{\text{ReAct}}$ in its average, while FaithAct omits it. Removing a strictly sub-maximal element from an arithmetic mean strictly increases that mean, provided the remaining elements are not all equal to that element. Since $F_{\text{step}}(s_{g^\star}^{\text{ReAct}}) < 1$ by assumption, this condition holds. Therefore the average step score of FaithAct is strictly higher than that of ReAct.

In both cases,
\[
F_{\text{chain}}\!\left(R^{\text{FaithAct}}\right)
>
F_{\text{chain}}\!\left(R^{\text{ReAct}}\right).
\]
\end{proof}

Empirically (Table 1), we observe that the inequality is typically strict, consistent with Corollary 1: whenever ReAct produces at least one perceptually ungrounded step, FaithAct either corrects it using verified evidence or removes it rather than propagating hallucinated content.